\documentclass[%
 aip,
 amsmath,amssymb,
 reprint,%
]{revtex4-1}

\usepackage{graphicx}
\usepackage{dcolumn}
\usepackage{bm}

\usepackage[utf8]{inputenc}
\usepackage[T1]{fontenc}
\usepackage{mathptmx}
\usepackage{etoolbox}
\usepackage{booktabs}
\usepackage{hyperref}
\usepackage{enumitem}
\usepackage{algorithmic}

\makeatletter
\newcounter{algorithm}
\renewcommand{\thealgorithm}{\arabic{algorithm}}

\newenvironment{algorithm}[1][tb]
  {\par\vspace{1em}
   \hrule
   \vspace{0.5em}
   \refstepcounter{algorithm} 
   \def\caption##1{
       \noindent\textbf{Algorithm \thealgorithm}~~##1 \par
       \vspace{0.5em}
       \hrule
       \vspace{0.5em}
   }
  }
  {\vspace{0.5em}
   \hrule
   \vspace{1em}\par
  }
\makeatother

\newcommand{\textbfit}[1]{\textbf{\textit{#1}}}

\usepackage{amsmath}
\usepackage{amssymb}
\usepackage{mathtools}
\usepackage{amsthm}

\theoremstyle{plain}
\newtheorem{theorem}{Theorem}[section]
\newtheorem{proposition}[theorem]{Proposition}

\theoremstyle{definition}

\theoremstyle{remark}

\makeatletter
\def\@email#1#2{%
 \endgroup
 \patchcmd{\titleblock@produce}
  {\frontmatter@RRAPformat}
  {\frontmatter@RRAPformat{\produce@RRAP{*#1\href{mailto:#2}{#2}}}\frontmatter@RRAPformat}
  {}{}
}%
\makeatother

\begin{document}

\preprint{AIP/123-QED}

\title[Boltzmann Reinforcement Learning for Noise resilience in Analog Ising Machines]{Boltzmann Reinforcement Learning for Noise resilience in Analog Ising Machines}
\author{Aditya Choudhary}
\affiliation{Geomechanics and Geochemistry Department, Sandia National Laboratories, Albuquerque, NM, USA}
\author {Saaketh Desai}
\affiliation{Center for Integrated Nanotechnologies, Sandia National Laboratories, Albuquerque, NM, USA}

\author{Prasad Iyer*}
\affiliation{Center for Integrated Nanotechnologies, Sandia National Laboratories, Albuquerque, NM, USA}
\email{ppadma@sandia.gov}

\date{\today}

\begin{abstract}
Analog Ising machines (AIMs) have emerged as a promising paradigm for combinatorial optimization, utilizing physical dynamics to solve Ising problems with high energy efficiency. However, the performance of traditional optimization and sampling algorithms on these platforms is often limited by inherent measurement noise. We introduce BRAIN (Boltzmann Reinforcement for Analog Ising Networks), a distribution learning framework that utilizes variational reinforcement learning to approximate the Boltzmann distribution. By shifting from state-by-state sampling to aggregating information across multiple noisy measurements, BRAIN is  resilient to Gaussian noise characteristic of AIMs. We evaluate BRAIN across diverse combinatorial topologies, including the Curie-Weiss and 2D nearest-neighbor Ising systems. We find that under realistic 3\% Gaussian measurement noise, BRAIN maintains 98\% ground state fidelity, whereas Markov Chain Monte Carlo (MCMC) methods degrade to 51\% fidelity. Furthermore, BRAIN reaches the MCMC-equivalent solution up to 192x faster under these conditions. BRAIN exhibits $\mathcal{O}(N^{1.55})$ scaling up to 65,536 spins and maintains robustness against severe measurement uncertainty up to 40\%. Beyond ground state optimization, BRAIN accurately captures thermodynamic phase transitions and metastable states, providing a scalable and noise-resilient method for utilizing analog computing architectures in complex optimizations.
\end{abstract}

\maketitle


\section{Introduction}

Analog Ising Machines (AIMs) represent a breakthrough in combinatorial optimization,  utilizing physical dynamics to minimize energy with $10^3-10^6\times$ speedups over digital processors. \cite{al2025programmable,mohseni2022ising,Prabhu:20,hua2025integrated,si2024energy,honjo2021100}
While digital emulators compute energy through sequential logic cycles, AIMs leverage physical dynamics to resolve thousands of interactions in a single computational step. Analog systems offer a fundamental advantage beyond computational speed: they naturally explore the full solution landscape, not just an isolated optima.
To fully harness this potential, algorithms must be engineered for the specific operating reality of high-speed analog cores: \textbf{inherent measurement noise} (3-10\%) and \textbf{sequential sampling interfaces} that prioritize rapid, iterative queries over massive parallel batches \cite{gao2024photonic}.
These hardware constraints create a fundamental mismatch for existing sampling and optimization algorithms. 
Markov Chain Monte Carlo (MCMC) \cite{metropolis1953equation, hastings1970monte} methods are iterative but rely on precise reward (energy, $\Delta E$) differences to maintain detailed balance; when measurement noise ($\sigma$) approaches $\Delta E$, the acceptance criterion becomes effectively random, and convergence fails. Conversely, state-of-the-art deep learning based combinatorial optimization solvers (e.g. diffusion models \cite{sanokowski2024diffusion}, GNNs \cite{schuetz2022combinatorial}) typically assume noiseless measurements and rely on heavy digital training and inference. This introduces latency that exceeds analog reward time, negating AIM's physical speed advantage.

We introduce \textbf{BRAIN (Boltzmann Reinforcement for Analog Ising Networks)}, an algorithm explicitly designed to bridge the gap between noisy, sequential hardware and robust sampling.
BRAIN shifts the paradigm from traversing the landscape state-by-state (MCMC) to variational inference with noisy oracles.
It learns a probabilistic policy parametrized by a factorized Bernoulli distribution to approximate the true Boltzmann distribution.
This approach transforms the  hardware's stochasticity from a liability into a learning signal, using policy gradients to naturally aggregate noisy measurements into stable distributional parameters.
Beyond simple ground state optimization, BRAIN accurately tracks thermodynamic phase transitions and identifies metastable states, maintaining sampling fidelity across temperature regimes and noise levels.
Unlike digital algorithms that may get trapped in brittle local optima, BRAIN navigates through the analog system's intrinsic fluctuations, delivering solutions that are physically robust to noise - a critical feature for applications like drug discovery where stability is vital \cite{taillard2022short}.

\noindent \\
\textbf{Main Contributions:}
\begin{enumerate}[leftmargin=12pt,labelsep=5pt]
    \item \textbf{Noise-Resilient Optimization:} By averaging out measurement noise via gradient accumulation, BRAIN delivers robust performance where traditional methods collapse. The algorithm maintains 98\% ground state fidelity under realistic 3\% noise conditions where standard MCMC degrades to 51\%.
    
    \item \textbf{High-Throughput Efficiency:} To ensure the time-to-solution is driven by the rapid analog oracle rather than digital latency, BRAIN compresses the optimization landscape from the full $\mathcal{O}(2^N)$ configuration space to $\mathcal{O}(N)$ variational parameters. This lightweight parameterization enables rapid updates with negligible digital overhead.
    
    \item \textbf{Scalable Variational Sampling:} Instead of expending computational resources to model the entire loss (energy) landscape of the combinatorial optimization problem, BRAIN uses variational optimization with low-latency models, leveraging the ``mode-seeking'' nature of reverse-KL divergenceoptimization to learn the key features of the Boltzmann distribution at finite temperature. This enables BRAIN to efficiently sample low-energy configurations and scale with $\mathrm{O}(N^{1.55})$ to 65,536 spins and converge 192$\times$ faster than MCMC. 
\end{enumerate} 

\section{Problem Description}
Our goal is to develop an efficient, scalable optimization and sampling algorithm that is resilient to noise inherent to analog Ising machines.

\noindent \\
\textbf{Analog Ising Machines}: Analog Ising machines are characterized as machines that use physical dynamics (optical parametric oscillators, or electric oscillators) to solve combinatorial optimization problems \cite{gao2024photonic} by framing them as solutions to an Ising model \cite{lucas2014ising}.
The Ising model, a cornerstone of statistical physics, is described by an Ising Hamiltonian $H(\mathbf{x}) = -\sum_{i<j} J_{ij} x_i x_j - \sum_i h_i x_i$, where $\mathbf{x} = (x_1, x_2, \ldots, x_N)$ with $x_i \in \{\pm 1\}$ represents a spin configuration, $J_{ij}$ are pairwise coupling strengths, and $h_i$ is the local field.

\noindent \\
\textbf{Objective}: Our objective is to formulate an algorithm that efficiently samples states, and optimizes for states, under the intrinsic noise and stochasticity of AIMs.
This can be stated as learning parameters $\theta$ for a state generator model $q_\theta(x)$, such that $q_\theta(x)$ approximates the Boltzmann distribution $p(x) = \frac{e^{-\beta E(x)}}{\sum_x e^{-\beta E(x)}}$, with inverse temperature $\beta$.
This distribution learning problem, critical to identify metastable states for problem such as drug discovery \cite{taillard2022short}, can be easily extended to a combinatorial optimization problem, by varying $\beta$.
As $\beta \rightarrow \infty$, the configurations sampled by $q_\theta(x)$ would be the ground state solution to the optimization problem.

\noindent \\
\textbf{Constraints and Hardware-Compatibility}: We are specifically interested in designing algorithms to solve the above problem, while respecting the following:
\begin{enumerate}[leftmargin=12pt,labelsep=5pt]
    \item \textit{Scalability}: Our algorithm must scale to high dimensions, where ${\sum_x e^{-\beta E(x)}}$ is intractable. For instance, a  $N = 32 \times 32$ AIM has $2^{32\times32} \approx 10^{308}$ possible states.
    \item \textit{Data-Free, Sample Efficient Learning}: Since the target distribution is intractable, we need an algorithm that can learn $q_\theta(x)$ without training data from $p(x)$, and querying $q_\theta(x)$ as few times as possible.
    \item \textit{Low-Latency Evaluation}: AIMs are characterized by nano-to-micro second timescales to evaluate energies, i.e., scores to proposed solutions. Our algorithm must thus propose candidate solutions at this rate, and not be a computational bottleneck that prevents efficient use of the analog Ising machine.
    \item \textit{Non-Differentiable, Noisy Evaluations}: Unlike conventional combinatorial optimization, AIMs return a non-differentiable, noisy energy (score) $\tilde{E}(\mathbf{x})$ for a configuration $\mathbf{x}$. Each measurement is corrupted by intrinsic Gaussian noise: $\tilde{E}(\mathbf{x}) = E(\mathbf{x})(1 + \eta_x)$ where $E(\mathbf{x}) = H(\mathbf{x})$ is the true energy and $\eta_x \sim \mathcal{N}(0, \sigma^2)$ with $\sigma \in [0.03, 0.10]$, representing the 3--10\% relative measurement error characteristic of current photonic and electronic analog implementations. The noise specific to our AIM is characterized in Appendix \ref{expt_noise_characterization}.
\end{enumerate}

\section{Related Work}
BRAIN connects combinatorial optimization (ground state finding) with neural sampling and distribution learning, specifically addressing the constraints of noisy, non-differentiable analog hardware.

\subsection{Combinatorial Optimization: Ground State Finding}

Recent deep learning approaches for Ising-like combinatorial optimization problems, such as Schuetz et al. \cite{schuetz2022combinatorial} demonstrate that graph neural networks (GNNs) can solve max-cut and Ising problems with near-optimal performance, while Karalias \& Loukas \cite{karalias2020erdos} develop graph neural optimization approaches for problems such as graph clustering.
These methods rely on differentiable relaxations of the Hamiltonian and analytical gradient access.
In contrast, AIMs provide only sequential, black-box energy evaluations $\tilde{E}(x)$.
Similarly, Reinforcement Learning (RL) frameworks like COOL \cite{mills2020rl} and DIRAC \cite{fan2023searching} optimize annealing schedules or find ground states in spin glasses, but these methods assume exact energy measurements, incompatible with AIMs.

\subsection{Distribution Learning: Sampling the Boltzmann Distribution}
Standard MCMC \cite{metropolis1953equation, hastings1970monte} and Parallel Tempering \cite{earl2005parallel} use iterative proposal/rejection to sample the Boltzmann distributions (see Appendix \ref{mcmc_method} for details). While parallel tempering overcomes energy barriers via replica exchange, it requires simultaneous simulations that are difficult to implement on serial analog hardware. 
Furthermore, our experiments in Appendix \ref{parallel_temp} show that the replica exchange mechanism is highly sensitive to measurement noise, with detailed balance breaking down under the 3\% noise levels characteristic of current implementations of analog Ising machines.

Neural sampling algorithms like Torlai \& Melko \cite{torlai2016learning} demonstrate that restricted Boltzmann machines can accurately reproduce thermodynamic observables including energy, magnetization, and specific heat across phase transitions, establishing the viability of neural approaches for equilibrium statistical mechanics.
Autoregressive models \cite{wu2019solving, hibat2021variational} have also shown high-fidelity finite-temperature sampling in complex energy landscapes, modeling $q_\theta(x) = \prod_{i=1}^N q_\theta(x_i | x_{<i})$.
These training methods begin formulating the sampling task as minimizing the reverse KL-divergence between model distribution $q_\theta(x)$ and target Boltzmann distribution $p(x)$.
This is equivalent to minimizing the free energy of the system, as explained in the Methods section.
Our work also leverages this approach, while choosing $q_\theta(x)$ that is compatible with AIMs, avoiding models that have a substantial computational overhead relative to analog energy evaluation times.

More recent work has extended the KL-divergence formalism to models that sample from latent spaces, such as diffusion models \cite{sanokowski2024diffusion, sanokowski2025scalable}.
Sanokowski et al. show that minimizing KL divergence between a diffusion model $q_\theta(x)$ and the Boltzmann distribution, followed by gradual modification of $\beta$, achieves state-of-the-art results on combinatorial optimization problems.
The iterative denoising process requires hundreds of neural network evaluations per candidate $x$, which would introduce digital latency exceeding the nano-to-micro second timescales of analog hardware \cite{prabhu2020accelerating, honjo2021100}.
Additionally, the training procedures rely on exact energy evaluations for score matching and loss computation.
These limitations also apply to neural samplers designed using normalizing flow based models \cite{noe2019boltzmann, li2025deep} or designed by treating sampling as a continuous-time transport problem \cite{holderrieth2025leaps, ou2025discrete}, learning rate matrices of Continuous-Time Markov Chains using locally equivariant networks.

\subsection{Algorithms for Noisy Analog Ising Systems}

Limited prior work explicitly addresses noisy analog Ising optimization. Pierangeli et al. \cite{pierangeli2020noise} demonstrated that tuning detection noise can actually increase success probabilities in spatial-photonic systems, suggesting noise can be an algorithmic resource. 
Reifenstein et al. \cite{Reifenstein2021CIM} analyze coherent Ising machines with optical error-correction and feedback, highlighting all-to-all programmability and chaotic search dynamics.

\section{Methods}
We use policy gradient methods to solve our variational learning task in a hardware-in-the-loop setting, learning the Boltzmann distribution directly from noisy, non-differentiable sequential measurements.

\noindent \\
\textbf{Learning a Distribution by Minimizing KL Divergence}:
We use variational inference to learn $q_\theta(x)$ directly by minimizing the KL divergence between $q_\theta(x)$ and $p(x)$.
\begin{equation}
\theta^* = \arg \min_\theta 
\mathrm{KL}(q_\theta \,\|\, p) = \sum_x q_\theta(x) \log \frac{q_\theta(x)}{p(x)}
\end{equation}
For $p(x)$ equal to the Boltzmann distribution, using standard definitions, we can write:
\begin{equation}
\label{KL_equation}
\mathrm{KL}(q_\theta \,\|\, p) = -H(q_\theta(x)) + \beta \, \mathbb{E}_{q_\theta}[E(x)] + \log Z
\end{equation}
where $H(q_\theta)$ is the entropy of the model and $Z = \sum_x e^{-\beta E(x)}$ is the partition function.
Since $\rm Z$ is a summation over all possible system states, it is a constant (which we cannot evaluate) that does not affect our optimization to learn $q_\theta(x)$.
Therefore, an equivalent objective function is:
\begin{equation}
\label{loss_function_equation}
\theta^* = \arg \min_\theta L(\theta) = -H(q_\theta(x)) + \beta \, \mathbb{E}_{q_\theta}[E(x)]
\end{equation}

\noindent \\
\textbf{Relating KL Divergence to a Helmholtz Free Energy}: The objective above has a natural interpretation in statistical mechanics, as the minimization of the Helmholtz free energy (F) of an equivalent physical system that is kept at a constant inverse temperature $\beta$.
\begin{equation}
F = \mathbb{E}_{q_\theta}[E(x)] - T H(q_\theta(x)) = U - TS
\end{equation}
where $U$ is the expected internal energy, $S$ the entropy, and $T = 1/\beta$ the temperature.
Minimizing the KL divergence is therefore equivalent to minimizing the free energy, and the equilibrium distribution of configurations is given by the Boltzmann distribution.
This minimization of free energy, or the ``reverse'' KL divergence, is also the starting point for other variational approaches to learning intractable distributions \cite{Torlai2016Thermo, wu2019solving, sanokowski2024diffusion, li2025deep}, as mentioned in Related Work.

\textbf{Learning with Non-Differentiable Rewards}:
One method of minimizing the free energy in Eq. \ref{loss_function_equation} is to use the score-function trick \cite{williams1992simple}, also known as the REINFORCE algorithm, a variant of policy-gradient reinforcement learning.
This method allows us to estimate gradients of an expectation over a parameterized distribution without requiring differentiability of the reward (here, $E(x)$).
This method estimates the gradient of the loss function (the free energy) as:
\begin{equation}
    \nabla_\theta L = \mathbb{E}_q[\beta E(x)\nabla_\theta \log (q_\theta(x))] - \nabla_\theta H(q_\theta(x))
\end{equation}
A common technique to reduce the variance of the REINFORCE gradient estimator is to subtract a baseline from the reward $E(x)$.

\noindent \\
\textbf{Hardware-Compatible Distributions}: To operate effectively within the AIM, the variational family $q_\theta(x)$ must satisfy two strict hardware constraints: (1) Noise robustness - The model must be learnable via high-variance gradient estimates derived from the noisy energy measurements (2) Low latency - The inference step must match the nano to micro second timescale of the reward inference on the AIM.

Driven by these constraints, we avoid the use of deep neural samplers (autoregressive models, diffusion models), and instead choose $q_\theta(x)$ to be a fully-factorized Bernoulli distribution over individual spins.
\begin{equation}
    q_\theta(x) = \prod_{j=1}^{N} m_j^{\frac{1+x_j}{2}} (1-m_j)^{\frac{1-x_j}{2}}, \quad m_j \in [0,1]    
\end{equation}
where each spin generator is a Bernoulli distribution with local variational parameters $m_j$.

This formulation converts the problem of generating an $\sqrt{N} \times \sqrt{N}$ spin configuration into generating $N$ local variational parameters $\{ m_j \}$, which comprise the parameter vector $\theta$.  
For this choice, the gradient of the loss with respect to each $m_j$ can be computed analytically, see equation below, and see Appendix \ref{gradient_expression_bernoulli} for details.
Appendix \ref{gradient_expression_gmm} also provides details for gradient expressions for other simple $q_\theta(x)$ such as Gaussian mixtures.
\begin{equation}
\nabla_{m_j} L = \mathbb{E}_{q_\theta} \Big[ \beta E(x) \Big( \frac{x_j - (2 m_j - 1)}{2 m_j (1 - m_j)}\Big)\Big] + \log \Big(\frac{m_j}{1 - m_j}\Big)
\end{equation}

\noindent \\
\textbf{Learning with Noisy Rewards}:
We minimize the free energy in Eq. \ref{loss_function_equation} with REINFORCE, using the noisy energy evaluation as the reward.
Specifically, we define the gradient of our loss function as:
\begin{equation}
    \nabla_\theta L \approx \frac{1}{S} \sum_{s=1}^{S} (r(x_s) - b) \nabla_\theta \log q_\theta(x_s) - \nabla_\theta H(q_\theta(x))
\end{equation}
where $r = \beta \tilde E(x_s)$, and where $b$ is a baseline defined as the average over all sampled configurations in a batch of size $S$.
 
We show that REINFORCE with baseline subtraction and noisy rewards asymptotically decreases the variance (as in the noiseless case), with the exact change in variance:
\begin{equation}
\Delta \mathrm{Var}_{\mathrm{noise}} = \frac{\sigma^2 \beta^2}{s^3} \left( \sum_{i=1}^s a_i \right) \left( \sum_{j=1}^s a_j \left[ 2 E(x_j)^2 - \overline{E^2} \right] \right)
\end{equation}
where $a_i = \nabla_\theta \log q_\theta(x_i)$ and $\overline{E^2} =\frac{1}{s}\sum_{k=1}^s E(x_k)^2$
see Proposition \ref{prop:noise_variance} in the Appendix.

Our complete algorithm is presented in Algorithm \ref{alg:score_function_training}.

\begin{algorithm}[tb]
\caption{BRAIN - Boltzmann Reinforcement for Analog Ising Networks}
\label{alg:score_function_training}
\noindent \textbf{Input}: Parameterized state generator $q_\theta(x)$, inverse temperature $\beta$, noise level $\sigma$ \\
\textbf{Parameters}: Samples per batch $S$, learning rate $\eta$ \\
\textbf{Output}: Optimal $\theta^*$ for $q_\theta(x)$
\begin{algorithmic}[1]
\FOR{each gradient step}
    \STATE Sample $S$ configurations $x_1, \dots, x_S \sim q_\theta(x)$
    \FOR{each sampled configuration $x_s$}
        \STATE Measure noisy energy: $\tilde{E}(x_s)$ and compute reward: $r(x_s) = \beta \tilde{E}(x_s)$
    \ENDFOR
    \STATE Compute baseline: $b = \frac{1}{S} \sum_{s=1}^{S} r(x_s)$
    \STATE Compute batch gradient estimate:
    \[
        \nabla_\theta L \approx \frac{1}{S} \sum_{s=1}^{S} (r(x_s) - b) \nabla_\theta \log q_\theta(x_s) - \nabla_\theta H(q_\theta(x))
    \]
    \STATE Update parameters: $\theta \gets \theta - \eta \nabla_\theta L$
\ENDFOR
\STATE \textbf{return} $\theta^*$
\end{algorithmic}
\end{algorithm}

\noindent \\
\textbf{Evaluation Metrics}

We evaluate BRAIN using a multi-dimensional suite of metrics that balance optimization quality, computational efficiency, and sampling fidelity. \textbfit{Ground State Fidelity} measures the success rate in identifying ground states, while \textbfit{Time to Solution} provides a hardware-agnostic assessment of sample efficiency via the total number of energy evaluations. For a fair comparison, we assume deployment on an AIM, where parallel evaluation of multiple spin states is prohibitive. We quantify computational overhead through \textbfit{Training Latency} (time per iteration and parameter count) and assess sampling diversity via the \textbfit{Effective Sample Size (ESS)} at critical temperature ($\beta = 0.4407$). This ESS metric, defined as:
\begin{equation}
\mathrm{ESS} = \frac{\left( \sum_{i=1}^N \exp(A_i) \right)^2}{\sum_{i=1}^N \exp(2 A_i)}
\end{equation}
where $A_i = -\beta E(x_i) - \log q(x_i)$, distinguishes models that explore equilibrium spaces effectively from those prone to weight concentration. Finally, we evaluate \textbfit{Scalability and Noise Resilience} by testing performance across system sizes up to $10^4$ spins and measuring resilience to noise ($\sigma$), noting that architectural constraints in current state-of-the-art solvers \cite{sanokowski2025scalable,fan2023searching,wu2019solving,holderrieth2025leaps,ou2025discrete} often preclude direct noise-resilience comparisons.

\section{Experiments}

We benchmark BRAIN on four problem  classes with progressively increasing energy-landscape complexity and high-dimensional interactions.

\subsection{Benchmarking on Low-Dimensional Energy Landscapes}

The first benchmark is a standard double well energy landscape $E(x) = A(x^2 - x^2_0)^2 + Bx$, where $A,B,x_0$ are constants determining the barrier of the double-well, the offset between the two energy minima, and the location of the minima respectively (see Figure \ref{test_case_noisy} (a)).
Also shown is the corresponding probability density $p(x) = e^{-\beta E(x)}/{Z}$, for two different inverse temperature $\beta$.
The addition of 10\% noise to the energy function produces a highly irregular energy landscape away from local minima.
This structural noise results in increased variability in the probability distribution $p(x)$, particularly evident in the high-temperature regime.
At low temperatures ($\beta = 0.3$), we see that the noise does not affect the likelihood of observing states significantly.
Note that for this problem, the state generator $q_\theta(x)$ was parameterized as a Gaussian Mixture Model (GMM) with two components, rather than the factorized Bernoulli distribution used for the discrete Ising experiments.
This allows the variational approximation to capture the bimodal nature of the continuous target distribution $p(x)$.
Figure \ref{test_case_noisy} (b) compares BRAIN and MCMC in their ability to recover $p(x)$ at both temperatures, where we find that both BRAIN and MCMC successfully recover $p(x)$ at high temperatures.
However, at low temperatures, MCMC oversamples the low probability state, while BRAIN captures $p(x)$ correctly.

The second benchmark is a one-dimensional six-spin case, as used in \cite{torlai2016learning}.
This benchmark allows us to measure performance on an Ising system where the partition function can be evaluated exactly, allowing us to visualize and compare $p(x)$ for BRAIN and MCMC, Figure \ref{test_case_noisy} (c-d).
We find that both BRAIN and MCMC adequately sample $p(x)$ at a low temperature where two states are more likely than other states.
Note that $p(x)$ is symmetric with respect to a sign inversion operation here, i.e., flipping the sign of all spins results in a state with equal $p(x)$.
Therefore, when BRAIN and MCMC recover one of the two likely states, they automatically recover the other likely state as well, via a sign inversion.
The noiseless versions of these benchmarks are shown in Appendix \ref{noiseless_benchmarks_lowd}.

\begin{figure}[h]
\begin{center}
\includegraphics[width=0.95\linewidth]{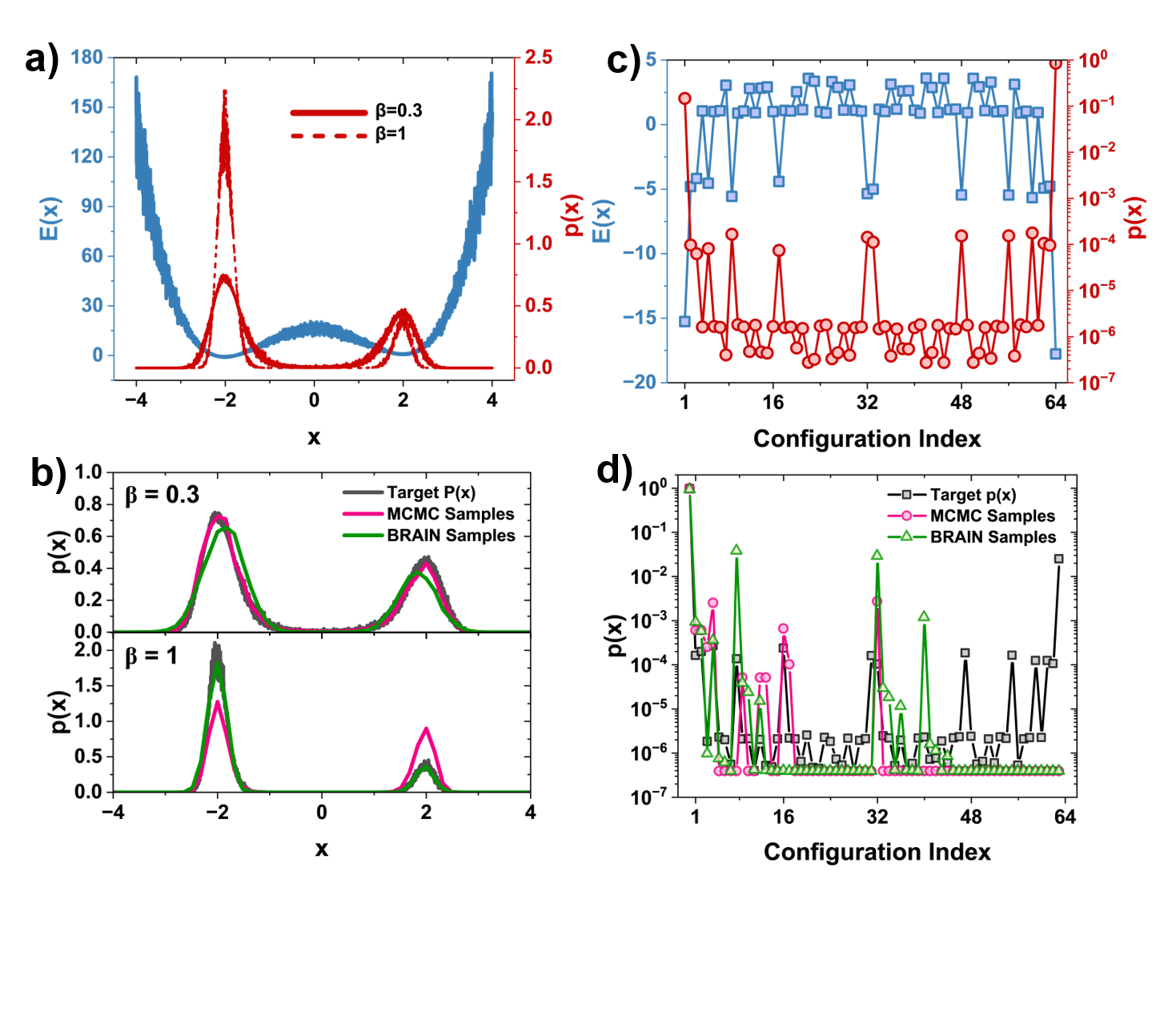}
\end{center}
\vspace{-6mm}
\caption{(a) A noisy double-well energy landscape $E(x)$, and associated probability $p(x)$ of observing state $x$. (b) Comparing BRAIN and MCMC to the ground truth $p(x)$ for two different temperatures. We find that both MCMC and brain perform adequately at high temperature but MCMC performs worse at low temperatures. (c) The energy landscape $E(x)$ and associated probability $p(x)$ for a one-dimensional six spin system. (d) Comparing BRAIN and MCMC to ground truth at a single, low temperature - we find both algorithms to be equally adequate at representing $p(x)$.}
\label{test_case_noisy}
\end{figure}

\subsection{Evaluating BRAIN on High-Dimensional Analog Ising Networks}
\subsubsection{Noisy Curie-Weiss Ising Hamiltonian}
Moving to higher-dimensional problems, we evaluate BRAIN using the 2D Curie-Weiss Ising Hamiltonian, a canonical Hamiltonian for understanding collective magnetic phenomena and phase transitions \cite{kochmanski2013curie}.
The Curie-Weiss Hamiltonian is given by: $H = -\frac{1}{2N}\sum_{i\neq j} J_{ij}\sigma_i\sigma_j $ where $\sigma_i$, $\sigma_j$, N, and $J_{ij} (= 1)$ represent spin $i$, spin $j$, total number of spins, and the spin-spin coupling strength, respectively.
This Hamiltonian represents an all-to-all coupled network topology, see Figure \ref{CW_Ising_Figure}(a), where every spin interacts with every other spin regardless of spatial distance. This fully connected mean-field model exhibits a well-characterized second-order phase transition at the critical temperature $T_c \approx J$, making it a standard benchmark for both optimization and sampling algorithms. 
\begin{figure*}[t!]
\begin{center}
\includegraphics[width=0.98\linewidth]{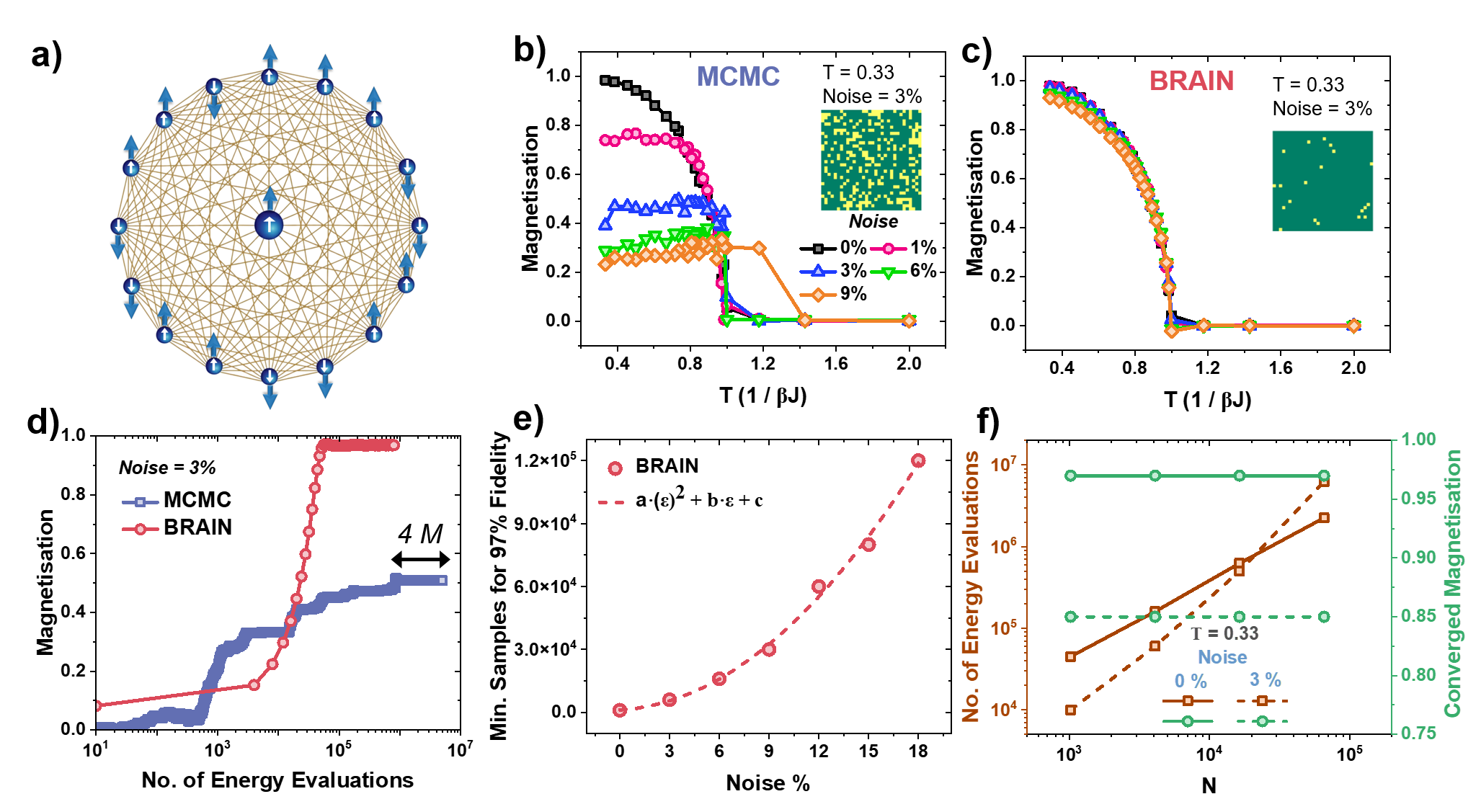}
\end{center}
\caption{(a) All-to-all coupled Ising network representing the Curie-Weiss Hamiltonian where every spin on the network is coupled to every other spin independent of the distance between the spins. Temperature-dependent magnetization profile for (b) MCMC and (c) BRAIN under various noise levels. The inset in (b) and (c) represents the ground-state spin state at T = 0.33 with 3\% noise. (d) MCMC and BRAIN convergence at T=0.33 under 3\% noise, shown in terms of the number of energy evaluations required to reach optimal solutions. (e) Minimum number of samples needed for BRAIN to maintain solution fidelity as noise increases. (f) Scalability analysis across system sizes (N) ranging from 1,024 to 65,536 spins. The left axis displays the number of energy evaluations at T=0.33 under both noiseless and 3\% noise conditions, while the right axis shows the corresponding magnetization values.}
\label{CW_Ising_Figure}
\end{figure*}

Under noisy evaluations, for a 32 $\times$ 32 system, MCMC (Figure \ref{CW_Ising_Figure}(b)) fails catastrophically, while BRAIN (Figure \ref{CW_Ising_Figure}(c)) demonstrates remarkable noise resilience in capturing the temperature-dependent magnetization profiles.
In the noiseless case (see Appendix \ref{noiseless_benchmarks_ising}), both algorithms accurately capture the thermodynamic landscape and converge to the theoretical ground state magnetization $|M| = 1$. 

However, even with minimal (1\%) noise, MCMC achieves only $|M| = 0.76$ instead of the theoretical $|M| = 1$ at low temperatures, see Figure \ref{CW_Ising_Figure}(b).
This failure arises because measurement noise approaches or exceeds the energy differences between competing configurations, causing the Metropolis acceptance probability to become effectively random, thereby violating the detailed balance condition required for correct Boltzmann sampling.
\cite{metropolis1953equation, hastings1970monte}.
This breakdown extends to advanced MCMC variants like parallel tempering, which enhances exploration through replica exchange, failing under  3\% analog noise levels (Appendix~\ref{parallel_temp}), confirming that the detailed balance mechanism is fundamentally incompatible with noisy energy evaluations.

Under realistic 3\% noise conditions, representative of current AIMs, BRAIN achieves $|M| = 0.98$ compared to MCMC's $|M| = 0.51$, representing 98\% versus 51\% \textit{fidelity} to the theoretical optimum. Representative spin configurations (green and yellow represent spin up and spin down states, respectively) at low temperature ($T = 0.33$) under 3\% noise are shown as insets in Figure~\ref{CW_Ising_Figure} (b-c). We find that BRAIN recovers the uniformly-aligned ground state (all green/all yellow), while MCMC fails. 

BRAIN's remarkable \textit{noise resilience} extends up to 9\% noise, achieves near-optimal magnetization and accurate  identification of $T_c$, see Figure \ref{CW_Ising_Figure}(c).
BRAIN also exhibits moderate degradation even at extreme 40\% noise, exhibiting only a 20\% degradation in magnetization, whereas MCMC suffers a severe 75\% loss in fidelity only at 9\% noise.
We note that expanding MCMC to noise-aware multiple energy averaging also performs poorly, see Appendix \ref{noise_degradation}, with 50 averages per sample necessary to increase the fidelity to 95\%.
This unprecedented performance under severe measurement uncertainty stems from BRAIN's policy gradient approach, aggregating noisy measurements into gradient updates.

Figure \ref{CW_Ising_Figure}(d) quantitatively compares \textit{time to solution} under 3\% noise conditions at T = 0.33, revealing a critical computational advantage of BRAIN's distribution learning approach. BRAIN reaches optimal solutions within $10^5$ energy evaluations, while MCMC remains trapped in suboptimal states even after $5\times10^6$ evaluations with $36\times$ more energy evaluations. 
We quantify the computational advantage of BRAIN over MCMC in Table \ref{table_MCMCvsBRAIN} as a function of increasing noise in the energy evaluations. 
This computational advantage arises from BRAIN's lightweight $\mathcal{O}(N)$ parameterization, enabling rapid updates with negligible digital overhead and being compatible with nano-to-micro second-timescale AIMs.

Figure \ref{CW_Ising_Figure}(e) highlights a key aspect of BRAIN's adaptive sampling strategy - maintaining high solution \textit{fidelity} under increasing noise requires a larger number of samples.
We observe a quadratic scaling in the additional samples required as the Gaussian noise level ($\epsilon$) increases. The data are well fitted by a second-order polynomial $ y = a\epsilon^2+b\epsilon+c$ where $a= 339.9,
b= 416.7, \& \ c= 1190.5$. 
At 1\% noise, BRAIN requires approximately $1.5\times$ more energy evaluations than the noiseless case. This overhead increases to $2\times$ at 3\% noise and $3\times$ at 6\% noise.
Appendix \ref{sample_size} documents ablation on sample size.
This scaling reflects the fundamental statistical requirement: higher noise demands more samples to achieve comparable signal-to-noise ratios in the gradient estimates.
Importantly, this overhead remains tractable because BRAIN’s $\mathcal{O}(N)$ update complexity scales favorably compared to the $\mathcal{O}(2^{N})$ configuration space that MCMC must explore, enabling efficient noise averaging while maintaining real-time operation.

\begin{table}[t]
\centering
\caption{BRAIN v/s MCMC - 1024 spin system (T = 0.33)}
\renewcommand{\arraystretch}{1.5} 
\resizebox{\columnwidth}{!}{
\begin{tabular}{|c|c|c|c|c|c|c|}
\hline
& \multicolumn{2}{c|}{\textbf{Fidelity}} & \textbf{Gain} & \multicolumn{2}{c|}{\textbf{Time to reach MCMC solution}} & \textbf{Acceleration} \\
\cline{2-3} \cline{5-6}
\textbf{Noise} & \textbf{MCMC} & \textbf{BRAIN} & \textbf{BRAIN/MCMC} & \textbf{MCMC} & \textbf{BRAIN} & \textbf{MCMC/BRAIN} \\
\hline
3\% & 51\% & 98\% & 1.9$\times$ & $1.0 \times 10^6$ & $5.2 \times 10^3$ & 192$\times$ \\
\hline
6\% & 37\% & 98\% & 3.3$\times$ & $1.1 \times 10^6$ & $4.5 \times 10^3$ & 244$\times$ \\
\hline
9\% & 29\% & 98\% & 3.9$\times$ & $1.0 \times 10^6$ & $3.5 \times 10^3$ & 285$\times$ \\
\hline
12\% & 27\% & 98\% & 4.5$\times$ & $0.9 \times 10^6$ & $2.4 \times 10^3$ & 408$\times$ \\
\hline
\end{tabular}
}
\label{table_MCMCvsBRAIN}
\end{table}

\begin{figure*}[t!]
\begin{center}
\includegraphics[width=\linewidth]{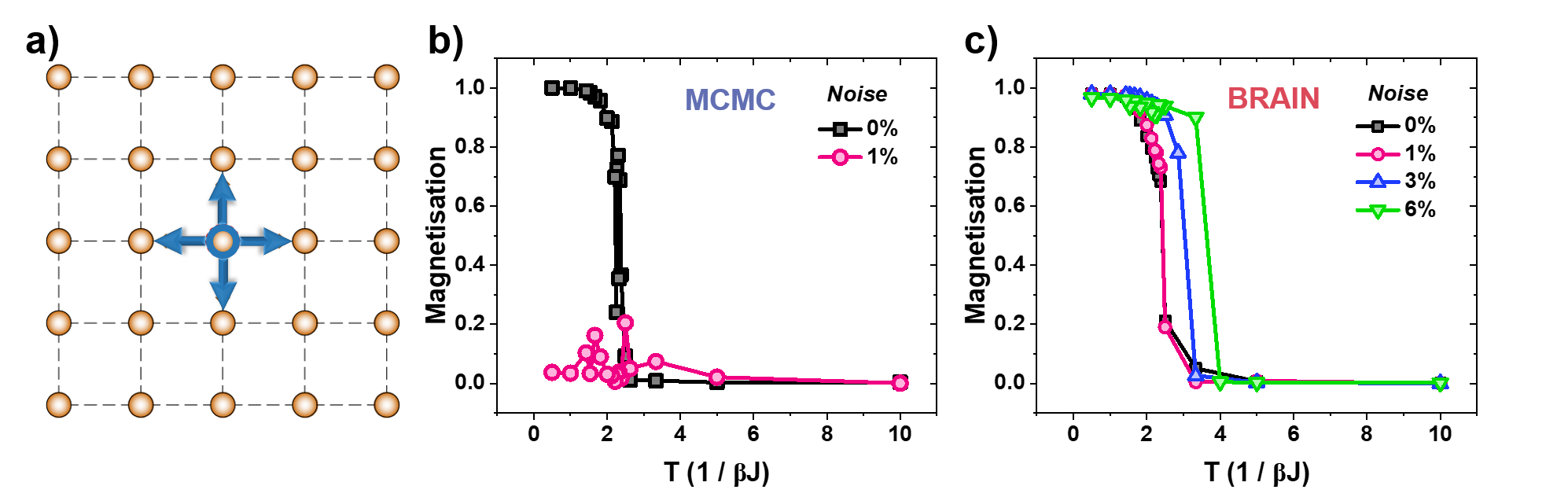}
\end{center}
\caption{(a) Schematic representing a 2D square array of spins emulating the classic Ising Hamiltonian with nearest-neighbor interactions (arrows). Temperature-dependent magnetization, comparing MCMC (b) and BRAIN (c) under increasingly noisy energy evaluations.}
\label{2D_Ising_Figure}
\end{figure*}

The scalability analysis, Figure \ref{CW_Ising_Figure}(f), demonstrates BRAIN's remarkable performance across system sizes ranging from 32×32 to 256×256 (65,536 spins), while revealing important trade-offs between noise tolerance and computational cost.
Under noiseless conditions, converged magnetization remains constant across all system sizes, indicating dimension-independent \textit{fidelity} (green solid line).
However, under 3\% noise, performance degrades modestly with increasing system size, with the magnetization dropping to 0.85 for the largest system, representing a 15\% deviation from the theoretical optimum (green dashed line).
Additionally, the \textit{time to solution}, measured in the of number of energy evaluations, exhibits distinct scaling behavior: $63.6N^{0.94}$ in the noiseless case (brown solid line) and $0.2N^{1.55}$ with 3\% noise (brown dashed line).
Note that scaling noisy evaluations assumes that we are converging to a fidelity of 0.85 while the noiseless scaling converges to a fidelity of 0.98.

Consequently, the noise-induced overhead increases systematically with system size, from 25\% additional evaluations at 1,024 spins to 100\% at 65,536 spins, reflecting the fundamental challenge of signal discrimination in high-dimensional noisy energy landscapes. This $\mathcal{O}(N^{1.55})$ scaling under noise remains polynomial and therefore compares favorably to the $\mathcal{O}(2^{N})$ complexity of exact partition function methods, demonstrating that BRAIN retains practical scalability even in noisy regimes relevant to AIMs.

\subsubsection{Lenz–Ising model: Generalization Beyond Mean-Field Ising Networks}
To demonstrate the generality of BRAIN beyond fully-connected topologies, we evaluate performance on the 2D nearest-neighbor Ising model, also known as the Lenz-Ising model, Figure \ref{2D_Ising_Figure} (a). This system serves as a foundational benchmark for NP-hard combinatorial optimization, with an energy defined by local interactions: $H = -J\sum_{i,j \in N(i)} \sigma_i \sigma_j$.
Figures \ref{2D_Ising_Figure}(b) and (c) confirm excellent agreement in \textit{fidelity} between BRAIN and MCMC under noiseless conditions, with both methods accurately capturing the critical temperature $T_c \approx 2.2$ predicted by Onsager's exact solution \cite{aharoni2000introduction}.
This validates that BRAIN reproduces thermodynamic phase transition behavior in systems with local spatial correlations, achieving accuracy comparable to established neural samplers \cite{wu2019solving}.
Notably, although the fully-factorized Bernoulli distribution assumes independence from spin topology, it successfully captures these phase transition properties because BRAIN treats the energy evaluation as a non-differentiable quantity rather than a differentiable, analytic quantity as used in naive mean-field approaches \cite{wu2019solving}.
This approach allows gradients to aggregate information from the true reward (energy) without making assumptions about the underlying topology or interactions between spins.

Under noisy conditions, the algorithms again show divergent behavior.
Figure \ref{2D_Ising_Figure}(b) reveals MCMC's complete breakdown in \textit{fidelity} even at 1\% noise, trapping the system in a fully disordered state ($M \approx 0$) across all temperatures, despite strong ordering expected at low T.
The noise renders the Metropolis acceptance criterion essentially random.
Figure \ref{2D_Ising_Figure}(c) demonstrates BRAIN's \textit{noise resilience} across noise levels from 0\% to 6\%.
At these noise levels, the algorithm maintains the characteristic magnetization-temperature curve and successfully identifies phase transitions.
However, a systematic trend emerges: the apparent $T_c$ shifts to higher values with increasing noise level at the rate of $\partial T_c/\partial \epsilon \approx 8.48J$. 

Along with the \textit{fidelity}, we also evaluate the \textit{sampling efficiency} of BRAIN at different noise levels using the effective sample size (ESS) metric, see Appendix \ref{ESS}.
BRAIN achieves an ESS of 0.005, which represents a 10$\times$ improvement over the best traditional MCMC sampler, validating our variational approach for discrete noiseless sampling.

\section{Conclusions and Outlook}
We have introduced BRAIN (Boltzmann Reinforcement for Analog Ising Networks), an algorithm that transforms the optimization of analog Ising machines (AIMs) from a state-to-state sampling problem into a noise-robust distribution learning task.
By leveraging policy gradients, BRAIN successfully bridges the gap between high-speed, noisy analog hardware and stable, scalable combinatorial optimization.

BRAIN maintains 98\% ground state fidelity under realistic 3\% Gaussian measurement noise, a regime where traditional MCMC methods degrade to 51\% \textit{fidelity}, thus showing \textit{noise resilience}.
BRAIN reaches MCMC-equivalent solutions 192x to 408x faster under varying noise levels, showing significant computational acceleration by leveraging an $\mathcal{O}(N)$ variational parameterization.
BRAIN demonstrates excellent \textit{scalability}, showing $\mathcal{O}(N^{1.55})$ scaling for systems up to 65,536 spins, remaining practical compared to the exponential complexity of exact methods.
Our algorithm \textit{generalizes} across diverse topologies, including all-to-all coupled Curie-Weiss models and NP-hard 2D nearest-neighbor Ising networks.
Finally, BRAIN features only 256 trainable parameters (for a 16x16 spin system), achieving \textit{real-time latency} of 0.052 seconds, see Appendix \ref{model_size}, ensuring that BRAIN does not bottleneck the nano-to-micro second AIMs.

Despite its advantages, BRAIN's current implementation has specific limitations that provide a roadmap for future refinement.
\textit{Expressivity in Frustrated Systems:} The use of a fully-factorized Bernoulli distribution assumes independence from spin topology, which may limit performance in complex landscapes.
Future work will incorporate more expressive architectures to capture higher-order correlations while maintaining real-time control constraints.
\textit{Non-Gaussian Hardware Realities:} While robust to Gaussian noise, physical hardware may exhibit systematic biases or temporal correlations.  We aim to develop robust policy gradient variants that utilize temporal filtering and non-parametric noise modeling to account for device-specific imperfections.

The underlying principle, policy gradients driven by noisy oracle feedback, naturally extends to Variational Quantum Algorithms \cite{Cerezo_2021} where measurement noise is a primary bottleneck. BRAIN establishes a foundation for utilizing physical computing architectures without the need for noise-free environments or exact energy evaluations.

\section*{Acknowledgments}
This work was supported by an Laboratory Directed Research \& Development (LDRD) project. This work was done at the Center for Integrated Nanotechnologies, an Office of Science user facility operated for the U.S. Department of Energy. This article has been authored by an employee of National Technology \& Engineering Solutions of Sandia, LLC under Contract No. DE-NA0003525 with the U.S. Department of Energy (DOE). The employee owns all right, title, and interest in and to the article and is solely responsible for its contents. The United States Government retains and the publisher, by accepting the article for publication, acknowledges that the United States Government retains a non-exclusive, paid-up, irrevocable, world-wide license to publish or reproduce the published form of this article or allow others to do so, for United States Government purposes. The DOE will provide public access to these results of federally sponsored research in accordance with the DOE Public Access Plan https://www.energy.gov/downloads/doe-public-access-plan.



\appendix

\section{Appendixes}

\subsection{Characterizing the noise in the analog Ising machine}
\label{expt_noise_characterization}

Figure \ref{Figure_S4}(a) presents direct experimental evidence of energy measurement fluctuations for a fixed spin configuration in a spatial photonic Ising machine (SPIM). Figure \ref{Figure_S4}(b) shows the corresponding distribution of measured energies across multiple measurements, revealing the Gaussian noise, characteristics inherent to photonic hardware. The bell-shaped distribution confirms that measurement noise follows a normal distribution. The width of the distribution provides direct measurement of the noise standard deviation $\sigma$, which appears consistent with the 3\% relative noise level used in this work.
\begin{figure}[!ht]
\begin{center}
\includegraphics[width=\linewidth]{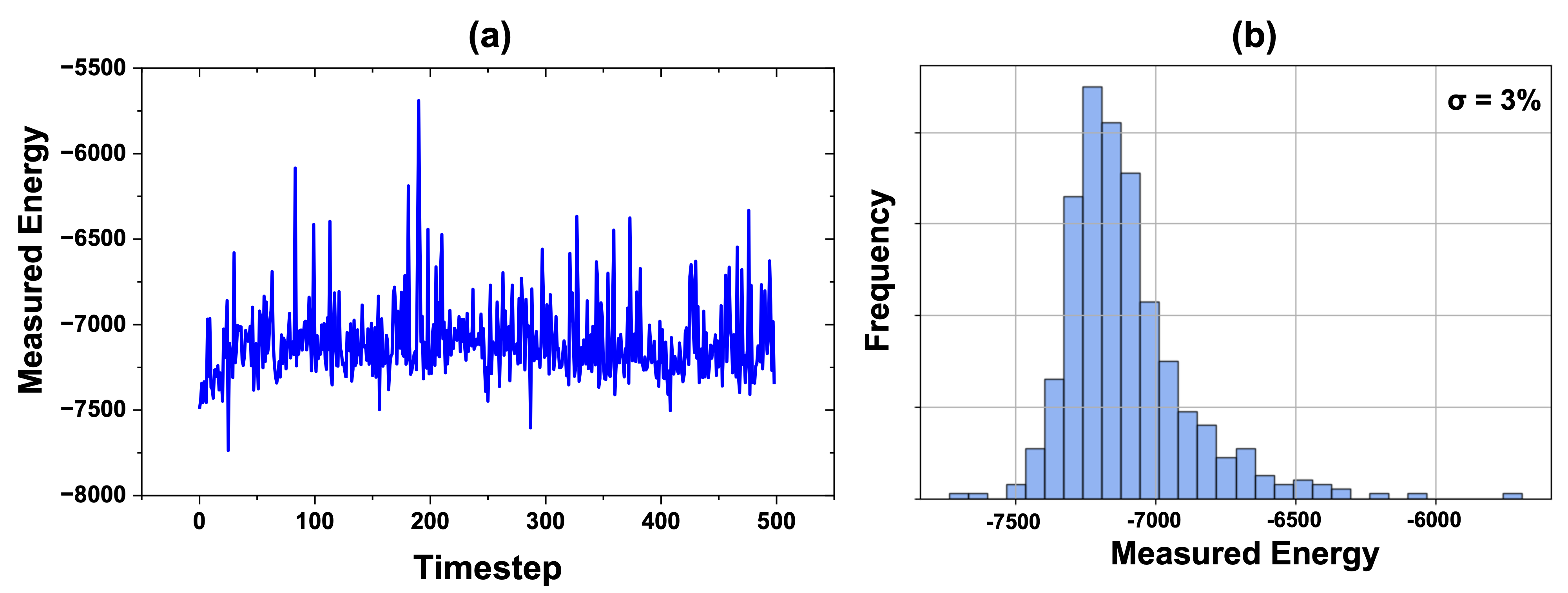}
\end{center}
\caption{(a) Fluctuation in energy observed for a fixed spin configuration using the Spatial Photonic Ising Machine (SPIM) across 500 experimental trials. (b) Histogram illustrating the distribution of measured energy values.}
\label{Figure_S4}
\end{figure}

\subsection{Markov Chain Monte Carlo approaches}
\label{mcmc_method}
Markov Chain Monte Carlo (MCMC) algorithms, such as the Metropolis-Hastings (MH) algorithm \cite{metropolis1953equation, hastings1970monte}, are a class of algorithms, designed to sample configurations from a target probability distribution $P(x)$.
MH in particular constructs a Markov chain with transition probabilities $P(x' | x)$ that satisfies detailed balance $P(x' | x) P(x) = P(x|x')P(x')$, ensuring $P(x)$ is stationary.
Factorizing the transition probabilities into a proposal probability $Q(x'|x)$ and an acceptance probability $\alpha(x',x)$, and enforcing detailed balance, for a target Boltzmann distribution $p(x) = \frac{e^{-\beta E(x)}}{\sum_x e^{-\beta E(x)}}$ \cite{metropolis1953equation}, and a symmetric (or uniform) proposal distribution, yields the familiar acceptance criterion:
\begin{align*}
\label{MH_equation}
\frac{P(x' | x)}{P(x|x')} &= \frac{Q(x'|x)\alpha(x',x)}{Q(x|x')\alpha(x,x')} = \frac{P(x')}{P(x)} \\& \alpha(x',x) = \min \Big(1, e^{-\beta (E(x') - E(x))} \Big)
\end{align*}
where $\beta$ is the inverse temperature.
While MH provides samples from the target stationary distribution, simulated annealing \cite{kirkpatrick1983optimization} gradually increases $\beta$ (lowers temperature) to concentrate samples near the global minimum, providing a stochastic global optimization procedure.
Numerous extensions of these fundamental methods have been proposed to improve sample efficiency and convergence, including Hybrid Monte Carlo \cite{duane1987hybrid}, and no U-turn sampling \cite{hoffman2014no} both of which attempt to improve acceptance rate and avoid inefficient walks in state space.
We point readers to more thorough reviews of this class of algorithms \cite{brooks2011handbook} for further background.
For the purposes of this work, we illustrate that this class of algorithms attempts to sample from a target distribution by accepting and rejecting moves in state space.

\subsection{Parallel tempering fails with noise on the Curie-Weiss Hamiltonian}
\label{parallel_temp}
We implemented parallel tempering \cite{earl2005parallel} as an advanced Monte Carlo baseline to investigate whether sophisticated sampling enhancements could overcome noise limitations in computing the Curie-Weiss Hamiltonian. The algorithm employed 30 replicas spanning temperatures from 0.33 to 2, with replica exchange attempts every 4,000 Monte Carlo steps over 400,000 total time steps per replica. While the method performed well under noiseless conditions, achieving efficient sampling across the temperature range with successful configuration exchanges between adjacent replicas (Figure \ref{Figure_PT}), it failed completely under 3\% noise conditions. All replicas remained trapped in disordered states ($M \approx 0$) despite the extended simulation time and frequent swap opportunities, with successful exchange events becoming extremely rare due to corrupted energy evaluations. This systematic failure across multiple temperature levels and extended sampling demonstrates that the fundamental limitation is not specific to simple Monte Carlo, but extends to advanced variants that still rely on single-measurement decision criteria incompatible with noisy energy evaluations. The contrast between noiseless success and noisy failure (Figure \ref{Figure_PT}) confirms that measurement uncertainty, rather than algorithmic sophistication, represents the critical bottleneck for traditional statistical mechanical approaches on analog computing platforms.
\begin{figure}[!ht]
\begin{center}
\includegraphics[width=0.8\linewidth]{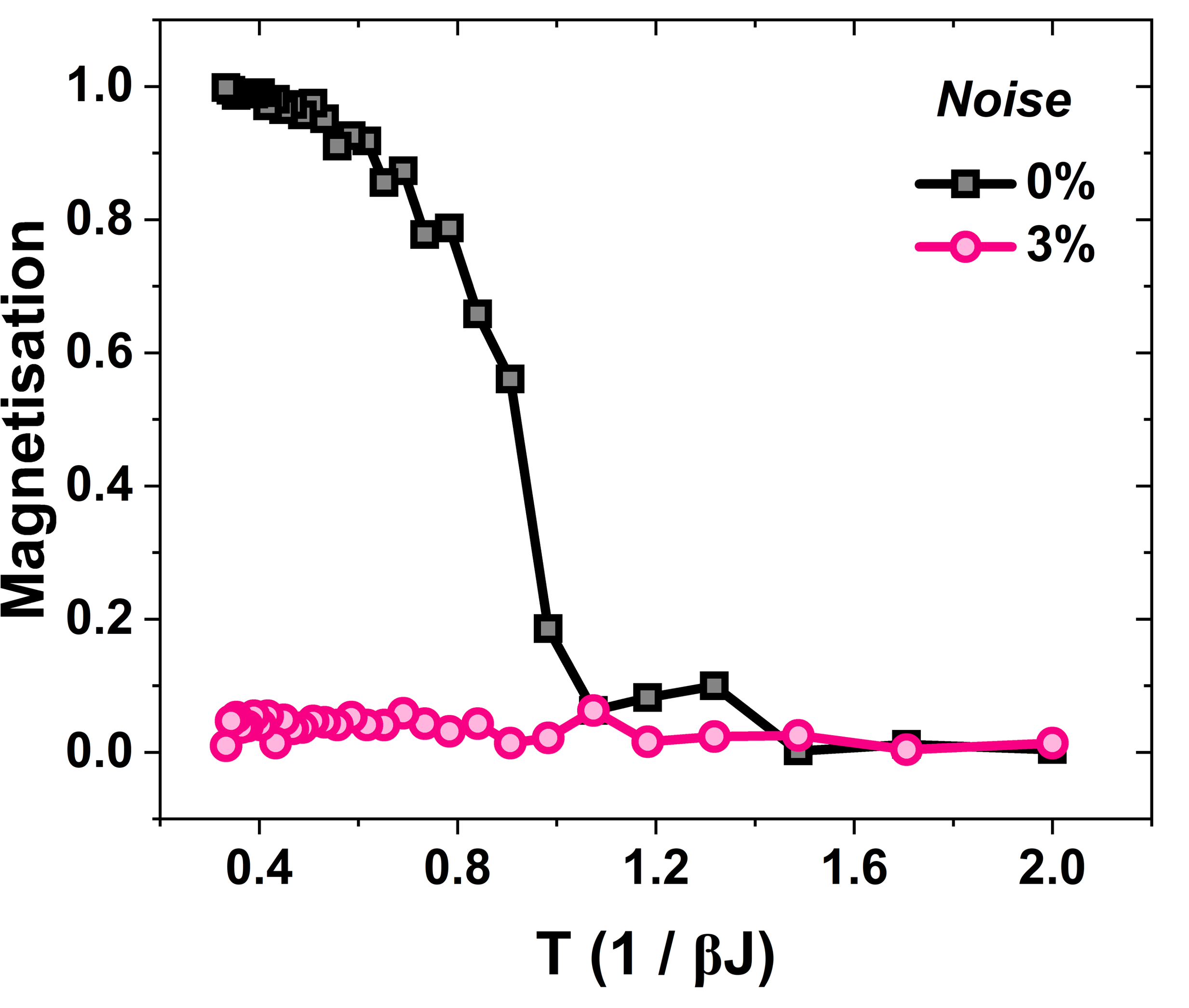}
\end{center}
\caption{Magnetization as a function of temperature for the parallel tempering algorithm, comparing performance under noiseless (0\% noise) and noisy (3\%) energy evaluations.}
\label{Figure_PT}
\end{figure}

\subsection{Gradient expressions for specific choices of $q_\theta(x)$}
\subsubsection{Bernoulli distribution}
\label{gradient_expression_bernoulli}
When $q_\theta(x)$ is the composition of independent Bernoulli distributions per spin, we can write:
\[
q_\theta(x) = \prod_{j=1}^{N^2} q_\theta(x_j), \quad q_\theta(x_j) = m_j^{\frac{1+x_j}{2}} (1-m_j)^{\frac{1-x_j}{2}}, m_j \in [0,1].
\]

Substituting this choice of $q_\theta(x)$ into our KL-based loss, we have:
\begin{align}
L(\theta) &= -H(q_\theta(x)) + \beta \, \mathbb{E}_{q_\theta}[E(x)] \nonumber \\
          &= -\sum_x \Big(\prod_j q_{\theta_j}(x_j)\Big) \log \Big(\prod_j q_{\theta_j}(x_j)\Big) + \beta \, \mathbb{E}_{q_\theta}[E(x)] \nonumber \\
          &= \sum_j \Big[ m_j \log m_j + (1 - m_j) \log (1 - m_j) \Big] + \beta \, \mathbb{E}_{q_\theta}[E(x)].
\end{align}
Given this loss function, we can efficiently learn $\theta$ using policy-gradient reinforcement learning or REINFORCE \cite{williams1992simple}. This approach is particularly useful when $E(x)$ is non-differentiable or noisy, as in measurements from an analog Ising machine, where backpropagation is not applicable.

For this choice, the gradient of the loss with respect to each $m_j$ can be computed analytically as:
\begin{equation}
\nabla_{m_j} L = \mathbb{E}_{q_\theta} \Big[ \beta E(x) \Big( \frac{x_j - (2 m_j - 1)}{2 m_j (1 - m_j)}\Big)\Big] + \log \Big(\frac{m_j}{1 - m_j}\Big)
\end{equation}

\subsubsection{Gaussian Mixture Model}
\label{gradient_expression_gmm}
For the continuous 1D benchmark (Figure 1), we parameterize the state generator $q_\theta(x)$ as a mixture of $K=2$ Gaussians:
\begin{equation}
    q_\theta(x) = \sum_{k=1}^K w_k \mathcal{N}(x; \mu_k, \sigma_k^2)
\end{equation}
where $\theta = \{w_k, \mu_k, \sigma_k\}_{k=1}^K$ are the learnable parameters, subject to $\sum w_k = 1$.
The gradient of the loss $L(\theta) = -H(q_\theta) + \beta \mathbb{E}_{q_\theta}[E(x)]$ with respect to the parameters is estimated via the score function estimator:
\begin{equation}
    \nabla_\theta L \approx \frac{1}{S} \sum_{s=1}^S \left( \beta (E(x_s) - b) \nabla_\theta \log q_\theta(x_s) \right) - \nabla_\theta H(q_\theta)
\end{equation}
where the score function $\nabla_\theta \log q_\theta(x)$ for a specific sample $x$ is given by:
\begin{align}
    \nabla_{\mu_k} \log q_\theta(x) &= \frac{w_k \mathcal{N}(x; \mu_k, \sigma_k^2)}{q_\theta(x)} \frac{(x - \mu_k)}{\sigma_k^2} \\
    \nabla_{\sigma_k} \log q_\theta(x) &= \frac{w_k \mathcal{N}(x; \mu_k, \sigma_k^2)}{q_\theta(x)} \left( \frac{(x - \mu_k)^2}{\sigma_k^3} - \frac{1}{\sigma_k} \right)
\end{align}

\subsection{Variance analysis under noise}
\label{variance_analysis}

\begin{proposition}
\label{prop:noise_variance}
Let $g^{(s)}(\theta)$ be the gradient estimator with batch size $s$, and $g^{(s,b)}(\theta)$ be the estimator with baseline subtraction $b = \frac{1}{s}\sum_{k=1}^s \tilde{E}(x_k)$. Under independent multiplicative Gaussian noise $\eta \sim \mathcal{N}(0, \sigma^2)$, the difference in noise-induced variance in the gradients is:
\[
\Delta \mathrm{Var}_{\mathrm{noise}}
= \frac{\sigma^2 \beta^2}{s^3}
\Big(\sum_{i=1}^s a_i\Big)
\Big(\sum_{j=1}^s a_j\,[2E_j^2-\overline{E^2}]\Big)
\]
where $\overline{E^2} =\frac{1}{s}\sum_{k=1}^s E(x_k)^2$, $a_k \equiv \nabla_\theta \log q_\theta(x_k)$
\end{proposition}

\begin{proof}
The standard estimator is defined as $g^{(s)}(\theta) = \frac{1}{s} \sum_{i=1}^s [ \beta E(x_i)(1+\eta_i) \nabla_\theta \log q_\theta(x_i) - \nabla_\theta H ]$. Since sampling and measurement noise are independent, the total variance in $g^{(s)}(\theta)$ decomposes into $\mathrm{Var}_{\mathrm{sample}} + \mathrm{Var}_{\mathrm{noise}}$. We focus on the noise component. For the standard estimator, the noise terms are independent:
\[
\mathrm{Var}_{\mathrm{noise}}[g^{(s)}] = \sum_{i=1}^s \mathrm{Var}\left[ \beta \eta_i E(x_i) \nabla_\theta \log q_\theta(x_i) \right] = \sum_{i=1}^s \sigma^2 a_i^2 E(x_i)^2
\]
where we substituted $a_i = \frac{\beta}{s} \nabla_\theta \log q_\theta(x_i)$.

For the baseline-subtracted estimator $g^{(s,b)}$, the noise term is $\sum_{i=1}^s a_i \delta_i$, where $\delta_i = \eta_i E(x_i) - \frac{1}{s}\sum_{j=1}^s \eta_j E(x_j)$ is the centered noise deviation. Due to the shared baseline, the noise terms are correlated:
\[
\mathrm{Var}_{\mathrm{noise}}[g^{(s,b)}] = \sum_{i=1}^s a_i^2 \mathrm{Var}[\delta_i] + \sum_{i \neq j} a_i a_j \mathrm{Cov}[\delta_i, \delta_j]
\]
Using the independence of $\eta_i$, the moments of $\delta_i$ are:
\begin{align*}
\mathrm{Var}[\delta_i] &= \sigma^2 E(x_i)^2 \left(1 - \frac{2}{s}\right) + \frac{\sigma^2}{s^2} \sum_{k=1}^s E(x_k)^2 \\
\mathrm{Cov}[\delta_i, \delta_j] &= -\frac{\sigma^2}{s}(E(x_i)^2 + E(x_j)^2) + \frac{\sigma^2}{s^2} \sum_{k=1}^s E(x_k)^2
\end{align*}

Subtracting $\mathrm{Var}_{\mathrm{noise}}[g^{(s,b)}]$ from $\mathrm{Var}_{\mathrm{noise}}[g^{(s)}]$ gives:
\begin{align*}
\Delta \mathrm {Var}_{\mathrm {noise}} &= \sum_{i=1}^s a_i^2 \sigma^2 E(x_i)^2 - \left[ \sum_{i=1}^s a_i^2 \mathrm{Var}[\delta_i] + \sum_{i \neq j} a_i a_j \mathrm{Cov}[\delta_i, \delta_j] \right] \\
&= \frac{\sigma^2}{s} \sum_{i=1}^s \sum_{j=1}^s a_i a_j (E(x_i)^2 + E(x_j)^2) \\&- \frac{\sigma^2}{s^2} \left(\sum_{k=1}^s E(x_k)^2\right) \left( \sum_{i=1}^s a_i \right)^2
\end{align*}
Factoring the double sum as $2(\sum a_i)(\sum a_i E(x_i)^2)$ and simplifying yields:
\[
\Delta \mathrm {Var}_{\mathrm {noise}} = \frac{\sigma^2}{s} \left( \sum_{i=1}^s a_i \right)
\left( \sum_{j=1}^s a_j \left[ 2 E(x_j)^2 - \overline{E^2} \right] \right)
\]
Substituting $a_i = \frac{\beta}{s} \nabla_\theta \log q_\theta(x_i)$ back into this expression yields the result in Proposition \ref{prop:noise_variance}. This result implies that while the baseline does not strictly reduce noise variance for every batch (as $\sum a_j \approx 0$ fluctuates), it effectively mitigates the variance contribution from the magnitude of $E(x)$.
Furthermore, with asymptotically large batch sizes, the change in variance with noise goes to zero, meaning that the change in variance is only due to batch size sampling, which we know reduces the variance in gradients.
\end{proof}

\subsection{Validating BRAIN on noiseless systems}

\subsubsection{Low-dimensional energy landscapes}
\label{noiseless_benchmarks_lowd}

\label{A1}
\begin{figure}[h]
\begin{center}
\includegraphics[width=\linewidth]{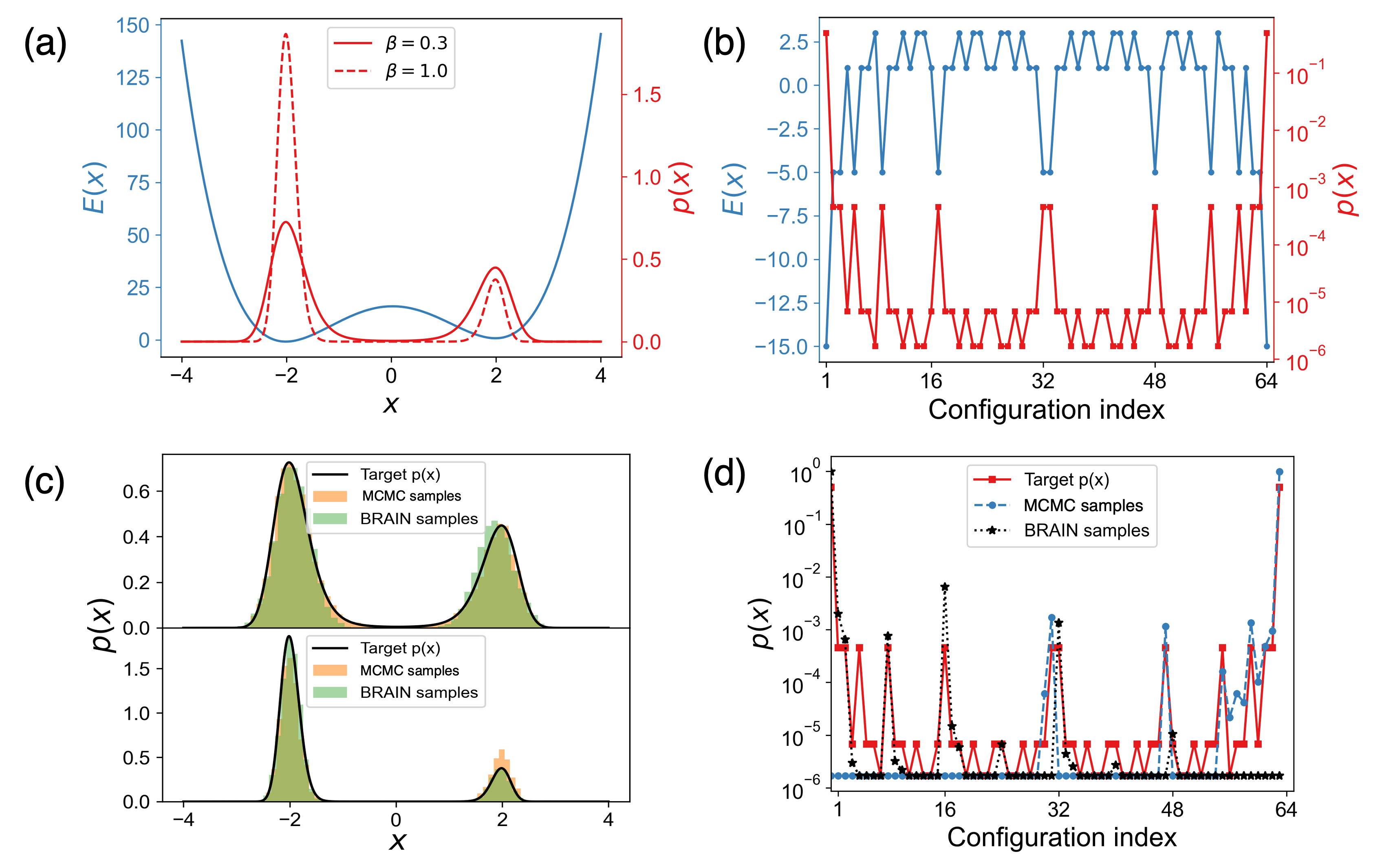}
\end{center}
\caption{(a) A noiseless double-well energy landscape $E(x)$, and associated probability density $p(x)$. (b) The energy landscape $E(x)$ and associated probability $p(x)$ for a one-dimensional six spin system. (c-d) Comparing BRAIN and MCMC to the ground truth $p(x)$ for two different temperatures.}
\label{test_case_noiseless}
\end{figure}

Fig. \ref{test_case_noiseless} shows the performance of BRAIN and MCMC on the test cases discussed in the main manuscript (see Figure \ref{test_case_noisy}).
We find that both BRAIN and MCMC are adequate at capturing $p(x)$ in the double-well energy case and one-dimensional six-spin case.

\subsubsection{High-dimensional Ising machine energy landscapes}
\label{noiseless_benchmarks_ising}

Figure \ref{Figure_S2}(a) shows the validation of BRAIN against MCMC methods for the Curie-Weiss model on a 32×32 lattice under noiseless conditions.
Both algorithms exhibit identical temperature-dependent magnetization profiles, accurately capturing the second-order phase transition at the critical temperature $T_c$ $\approx$ 0.87.
The magnetization- temperature curve confirms accurate thermodynamic behavior, with both methods achieving complete ferromagnetic ordering ($|M| = 1$) at low temperatures and complete disorder ($|M| \approx 0$) at high temperatures.
The excellent quantitative agreement across the entire temperature range validates that BRAIN correctly learns the underlying statistical mechanics, reproducing the Boltzmann distribution with the same fidelity as MCMC method.
The qualitative similarity between MCMC and BRAIN configurations at each temperature point confirms that BRAIN samples from the same statistical ensemble as MCMC (Figure \ref{Figure_S2}(b-c)), avoiding systematic biases that could lead to artificial ordering or disorder.
This ability to capture correct configurations at various temperatures is critical for applications that require accurate representation of the full Boltzmann distribution \cite{taillard2022short}, not just ground state optimization.

\begin{figure}[ht]
\begin{center}
\includegraphics[width=\linewidth]{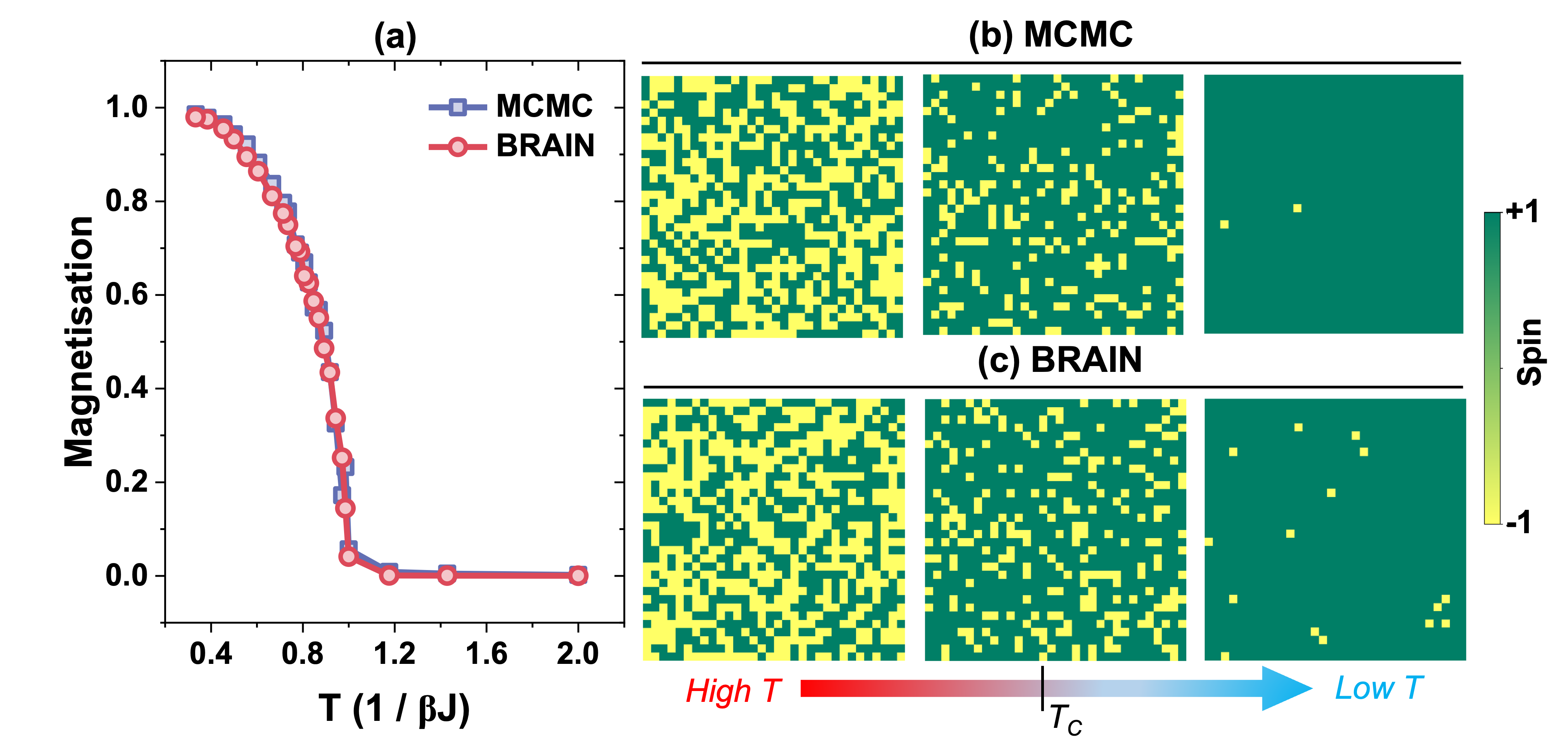}
\end{center}
\caption{(a) Temperature-dependent magnetization comparison between MCMC and BRAIN for the Curie-Weiss model on a 32×32 lattice. Representative spin configurations at three temperatures: $T>T_c$ (disordered), $T\approx T_c$ (critical fluctuations), and $T<T_c$ (ordered) for (b) MCMC and (c) BRAIN.}
\label{Figure_S2}
\end{figure}

\subsection{Ablation study I: noise resilience analysis}
\label{noise_degradation}

Figure \ref{Figure_S5} provides a systematic investigation of BRAIN performance across an extreme range of noise levels, from perfect (noiseless) conditions to severe 40\% noise that would render most algorithms completely ineffective. This analysis, conducted for the Curie-Weiss Hamiltonian, reveals BRAIN to be remarkably robust under conditions that far exceed realistic hardware specifications.
At low noise levels (up to 9\%), the algorithm maintains near-perfect performance, with magnetization-temperature curves virtually indistinguishable from the noiseless case.
The critical temperature remains accurately identified, and the transition sharpness is preserved, indicating that moderate noise levels do not pose a challenge to the algorithm.
Even under severe noise conditions (10-40\%), the algorithm continues to exhibit recognizable phase transition behavior, albeit with increased transition broadening and reduced low-temperature magnetization. The persistence of ordered phases under such extreme conditions demonstrates the fundamental robustness of BRAIN.
\begin{figure}[!htp]
\begin{center}
\includegraphics[width=\linewidth]{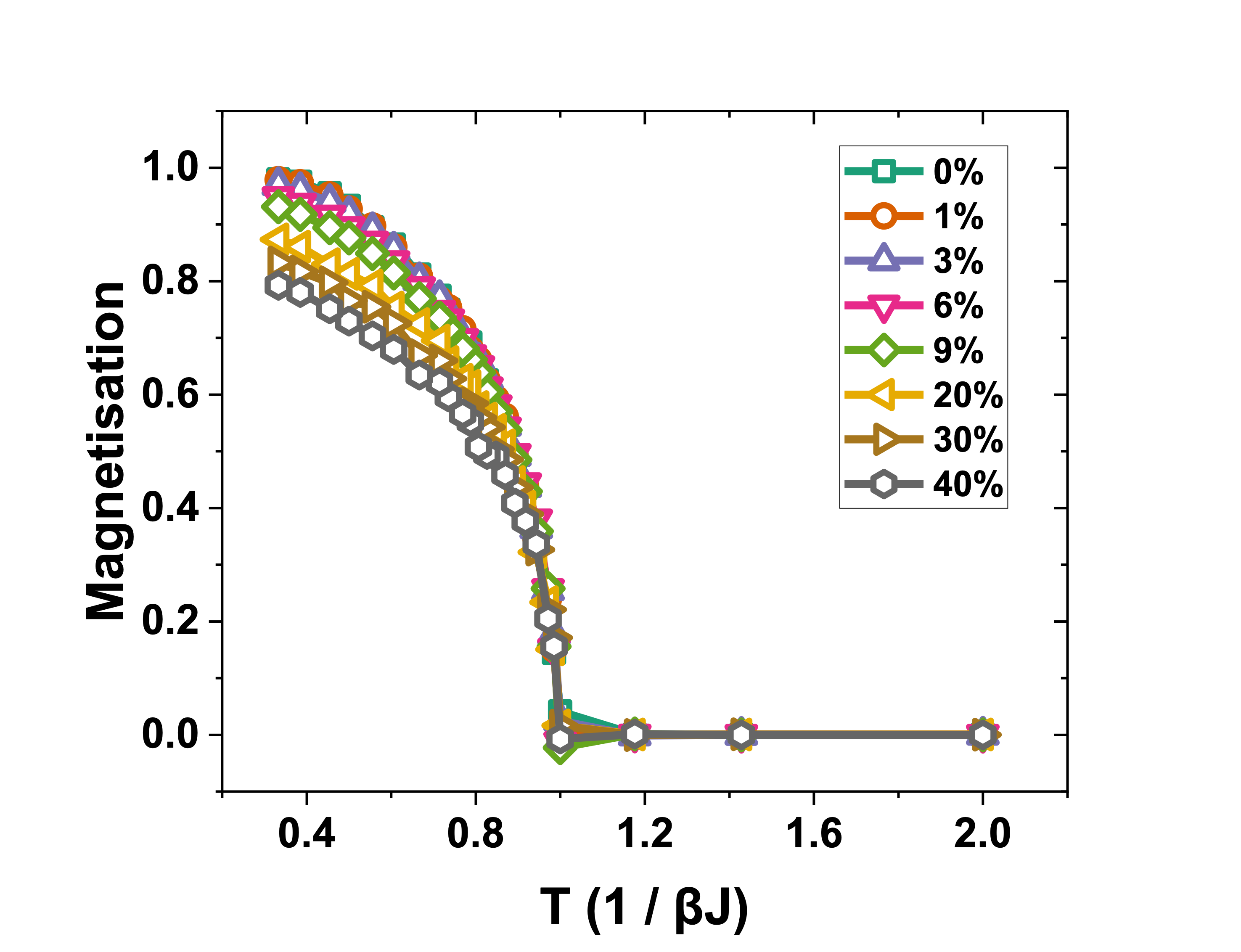}
\end{center}
\caption{Temperature-dependent magnetization profile for BRAIN under various noise levels ranging from 0 to 40\% for a 32x32 system}
\label{Figure_S5}
\end{figure}

Figure \ref{Figure_S6} provides a direct quantitative comparison between MCMC and BRAIN as a function of noise level in Curie-Weiss Hamiltonian at low temperature (T = 0.33).
The contrasting behavior of the two algorithms reveals fundamental differences in their noise handling capabilities. The MCMC curve shows catastrophic degradation even at minimal noise levels.
At just 1\% noise, the magnetization drops quickly from the theoretical optimum of $|M| = 1$ to approximately $|M| = 0.5$, representing a 50\% loss in optimal solution.
Further increases in noise level cause continued degradation, with the algorithm essentially failing to distinguish between random and ordered configurations at higher noise levels.
\begin{figure}[!ht]
\begin{center}
\includegraphics[width=\linewidth]{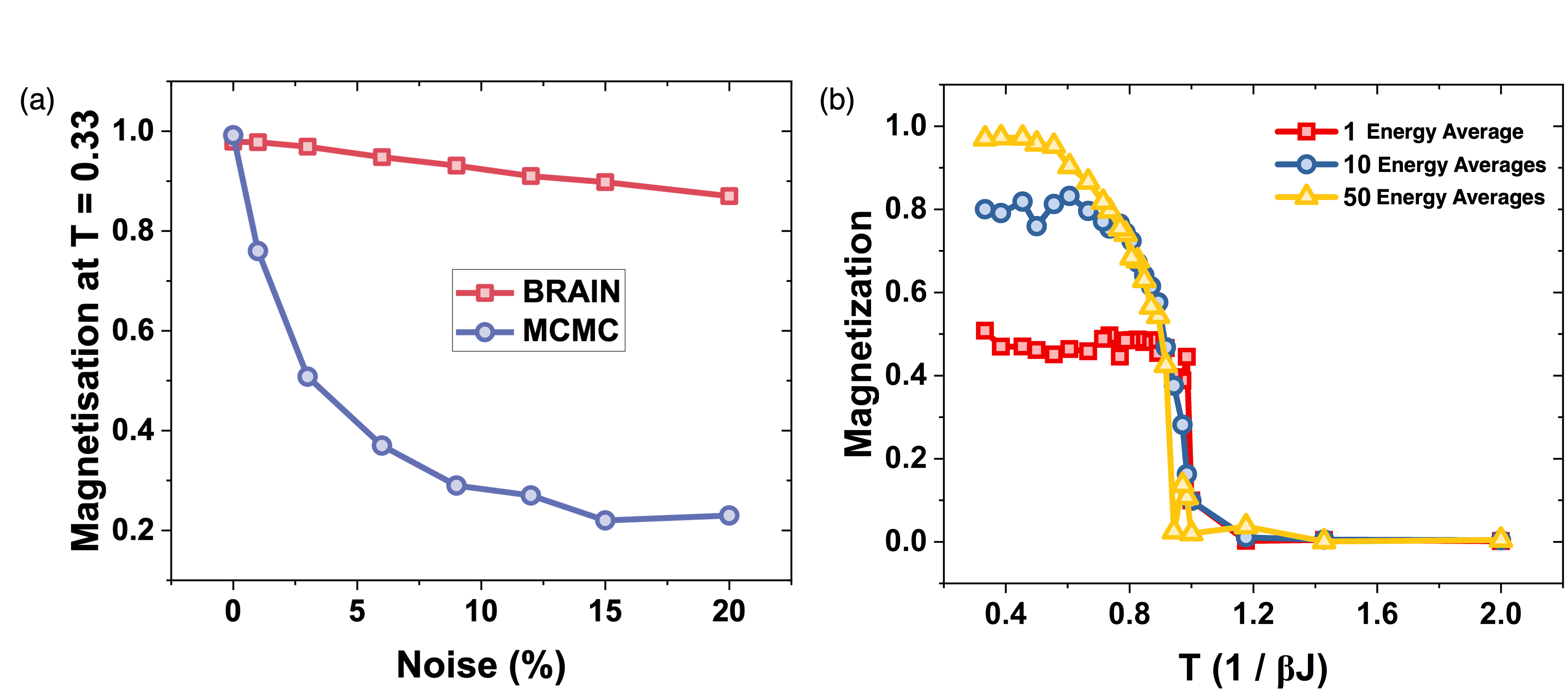}
\end{center}
\caption{(a) Comparing converged magnetization at low temperature (T=0.33) as a function of noise for MCMC and BRAIN. (b) Magnetization vs temperature with MCMC and varying levels of energy averaging to reduce noise.}
\label{Figure_S6}
\end{figure}
In contrast, the BRAIN maintains robust performance across the entire noise range. Even at 20\% noise, BRAIN achieves magnetization values above 0.8, representing less than 20\% degradation from the optimal solution.

\subsection{Ablation study II: optimal sample size}
\label{sample_size}
Figure \ref{Figure_S3} investigates the relationship between sample size and optimality for the Curie-Weiss model at low temperature (T = 0.33) on a 32×32 system. For very small sample sizes (n \textless 1,000), the algorithm shows suboptimal convergence, achieving magnetizations significantly below the theoretical optimum of $|M| = 1$. This poor performance reflects insufficient sampling for accurate gradient estimation in the BRAIN algorithm - the policy gradient estimates become too noisy to drive effective learning when based on too few samples.
Then with an increase in sample size, magnetization reaches the optimal value. For n \textgreater 1000, the curve plateaus near the theoretical optimum, indicating convergence to the ground state.
\begin{figure}[h]
\begin{center}
\includegraphics[width=\linewidth]{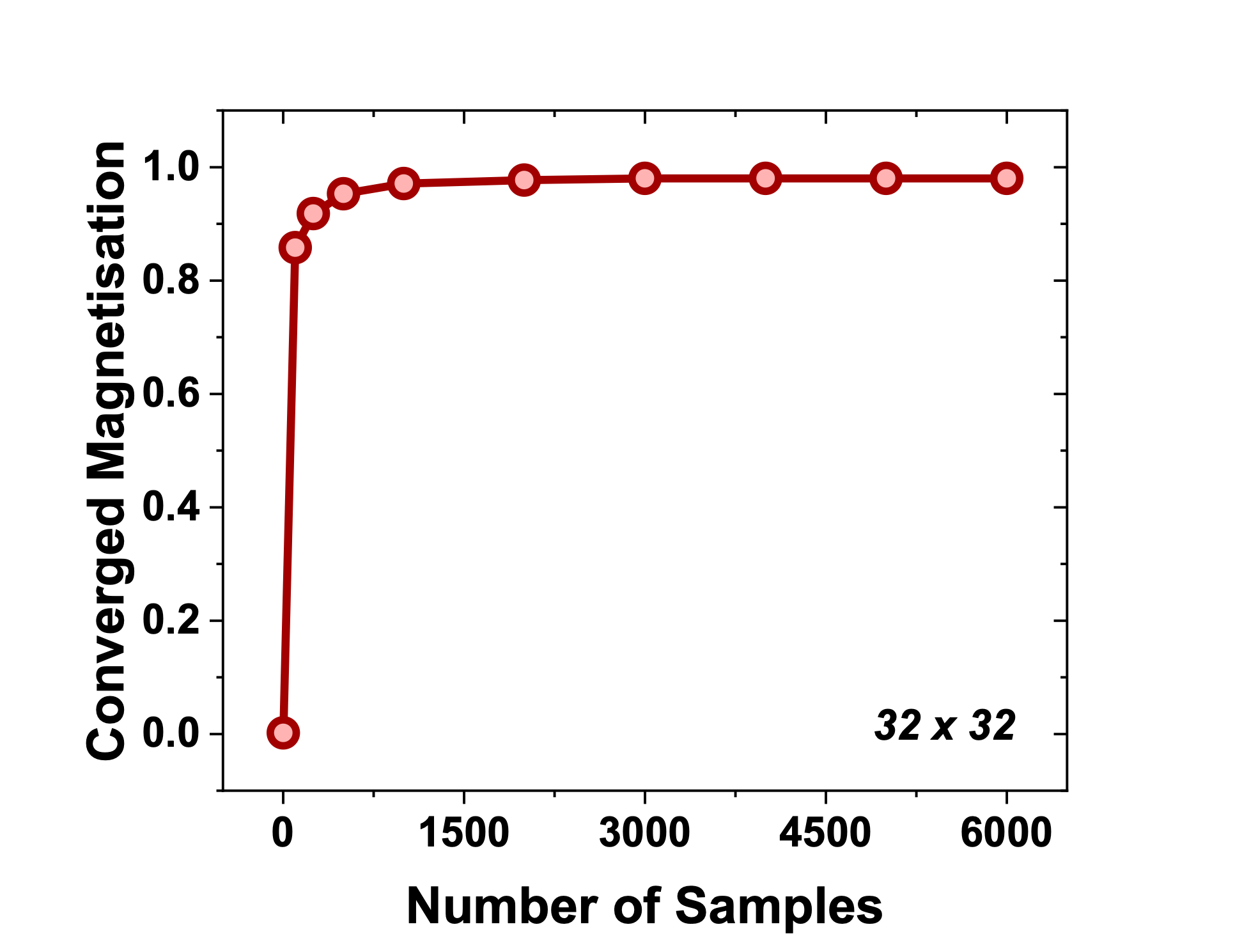}
\end{center}
\caption{Effect of number of samples on the optimality of Curie-Weiss solution at T = 0.33 for a 32x32 system.}
\label{Figure_S3}
\end{figure}

\subsection{Effective sample size comparison at critical temperature}
\label{ESS}
We evaluate the sampling efficiency of BRAIN against state-of-the-art sampling methods using the effective sample size (ESS) metric on the 2D Ising model at critical temperature ($\beta = 0.4407$). 

\begin{figure}[!ht]
\begin{center}
\includegraphics[width=\linewidth]{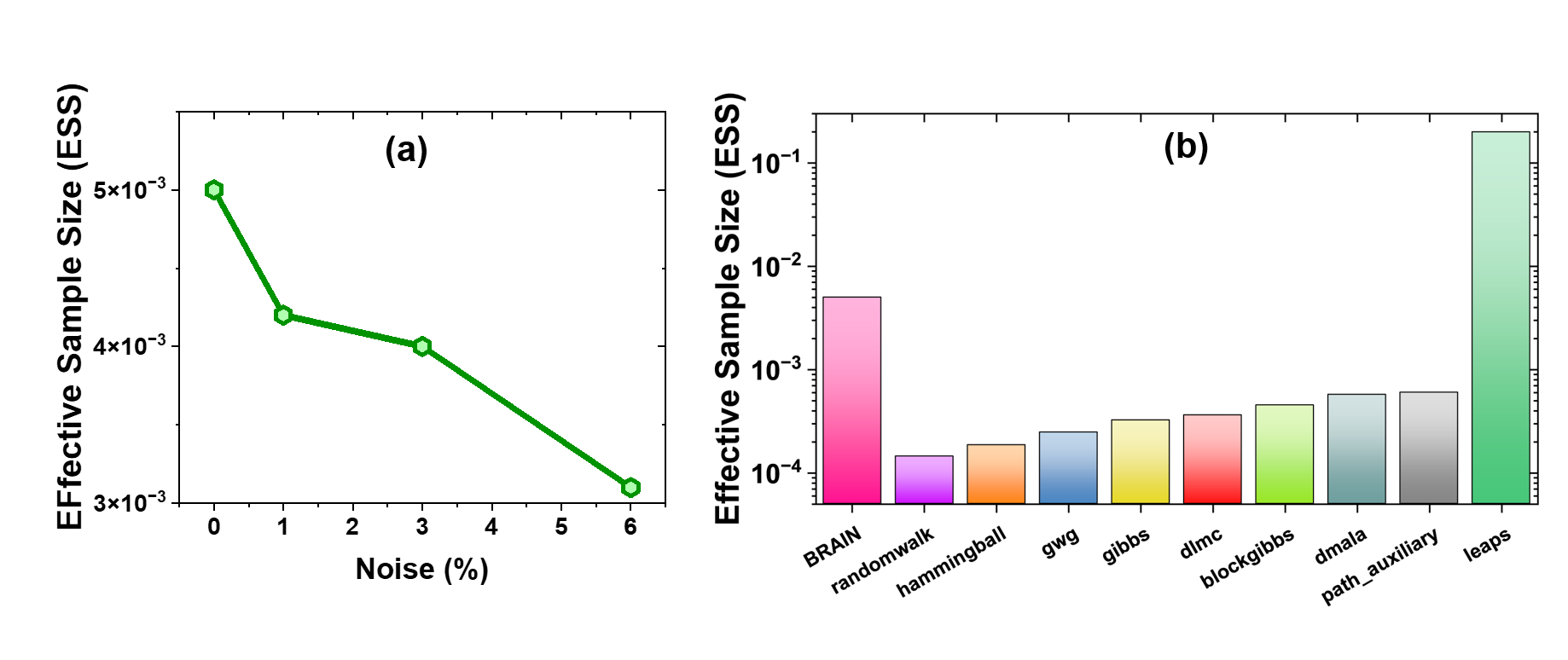}
\end{center}
\caption{Effective sample size (ESS) comparison across samplers. All methods are evaluated at the critical temperature.}
\label{Figure_ESS}
\end{figure}

We compare ESS of BRAIN against two categories of baselines: (1) traditional MCMC samplers from the DISCS benchmark \cite{goshvadi2023discs}, including Random Walk, Hamming Ball, Gibbs Within Gibbs (GWG), standard Gibbs, Discrete Langevin Monte Carlo (DLMC), Block Gibbs, and Discrete MALA (DMALA); and (2) modern learned samplers, including Path Auxiliary and LEAPS. All these methods are evaluated at critical temperature, representing the most challenging sampling regime where correlation lengths diverge, see Figure \ref{Figure_ESS}.

Traditional MCMC samplers exhibit extremely low ESS values, ranging from $1.44 \times 10^{-4}$ (Random Walk) to $5.71 \times 10^{-4}$ (DMALA).
Path Auxiliary achieves marginally better performance (ESS = $5.99 \times 10^{-4}$). LEAPS demonstrates a substantial improvement with ESS = 0.2 in their original work. This represents approximately a 350$\times$ improvement over the best MCMC baseline. On the other hand, BRAIN achieves an ESS of 0.005, which represents a 10$\times$ improvement over the best traditional MCMC sampler (DMALA), validating our variational approach for discrete sampling.
However, BRAIN's ESS is 100$\times$ smaller than LEAPS.
Such a significant gap between LEAPS and BRAIN is attributed to their fundamental methodological differences, which result in a mismatch between the learned  distribution and the Boltzmann target.
LEAPS, with an expressive diffusion model, learns a distribution $q_\theta(x)$ that is significantly closer to Boltzmann, compared to BRAIN. 
Samples from regions where $q_\theta(x)$ assigns high probability but the target (Boltzmann) assigns low probability receive small importance weights, leading to weight concentration, and thus inefficient sampling. 
However, we further document that our ESS drops from 0.005 to 0.003 with 6\% noise.
The reduction in ESS of other algorithms under noise is yet to be documented.

\subsection{Model size and latency comparison}
\label{model_size}
\begin{table}[h]
\centering
\caption{Comparison of model architectures and computational requirements for a 16×16 lattice\\}
\begin{tabular}{|l|l|l|}
\hline
\textbf{Models} & \textbf{Trainable Parameters} & \textbf{Time per Training Iteration (s)} \\
\hline
LEAPS & 330,000 & 20 \\
\hline
DiffUCO & 468,090 & - \\
\hline
DIRAC & 54,080 & 1.37 \\
\hline
BRAIN & 256 & 0.052 \\
\hline
\end{tabular}
\label{tab:models}
\end{table}

Table~\ref{tab:models} presents a comparison of four different models in terms of their latency, as evaluated by their architectural complexity and computational efficiency.
All experiments were conducted on CPU-only high-performance computing (HPC) environment with 2 Intel Xeon E5-2695 v4 CPUs (36 cores per node, 72 threads with hyperthreading) and 512 GB RAM. 
No GPU acceleration was provided, ensuring a fair comparison relevant to AIMs across models.

DiffUCO \cite{sanokowski2025scalablediscretediffusionsamplers,sanokowski2024diffusion} employs the largest architecture with 468,090 trainable parameters using a UNet-based diffusion model.
LEAPS \cite{holderrieth2025leaps} utilizes a dual neural network architecture comprising 330,000 trainable parameters.
Training was conducted with a batch size of 200 and 200 Markov Chain Monte Carlo (MCMC) steps per iteration, requiring approximately 20 seconds.
DIRAC  \cite{fan2023searching} implements a graph neural network (GNN) with 54,080 trainable parameters within an RL framework.
This model was trained with a batch size of 16, and each training iteration required approximately 1.37 seconds.
Overall, BRAIN represented the most parameter-efficient approach with only 256 trainable parameters, achieving the fastest training time of 0.052 seconds per iteration. This minimal architecture demonstrates that effective optimization can be achieved without the computational overhead of larger networks.

\section*{References}
\nocite{*}
\bibliography{aipsamp}

\end{document}